\documentclass[wcp]{jmlr}

\makeatletter
\def\set@curr@file#1{\def\@curr@file{#1}}
\makeatother
\usepackage[load-configurations=version-1]{siunitx}

\usepackage{rotating}
\usepackage{longtable}

\usepackage{booktabs}

\usepackage{mycommand}

\title[Bridging Ordinary-Label Learning and Complementary-Label Learning]{Bridging Ordinary-Label Learning and Complementary-Label Learning}

\author{\Name{Yasuhiro Katsura} \Email{katsura\_ani@ruri.waseda.jp}\\
\Name{Masato Uchida}\\
\addr Waseda University, Japan}

\begin{document}

\maketitle

\begin{abstract}
A supervised learning framework has been proposed for the situation where each training data is provided with a \textit{complementary label} that represents a class to which the pattern does not belong.
In the existing literature, \textit{complementary-label learning} has been studied independently from \textit{ordinary-label learning}, which assumes that each training data is provided with a label representing the class to which the pattern belongs.
However, providing a complementary label should be treated as equivalent to providing the rest of all the labels as the candidates of the one true class.
In this paper, we focus on the fact that the loss functions for one-versus-all and pairwise classification corresponding to ordinary-label learning and complementary-label learning satisfy certain \textit{additivity} and \textit{duality}, and provide a framework which directly bridge those existing supervised learning frameworks.
Further, we derive classification risk and error bound for any loss functions which satisfy additivity and duality.
\end{abstract}
\section{Introduction}\label{sec:introduction}

In the most common problem setting in supervised learning, the class to which the training data belongs is provided as a label, namely an \textit{ordinary label}.
There is increasing interest in generalizing the concept of ordinary labels \cite{chapelle2010mitpress, Jie2010nips, cour2011jmlr, natarajan2013nips, ishida2018nips}.
In this paper, we consider a problem setting where each training data is provided with a label, that is, a \textit{complementary label}, which specifies \textit{one} class that the pattern does not belong to.
The ordinary and complementary labels respectively reflect the information regarding a class to which the training data is
likely and unlikely to belong.

The learning framework was first proposed by \cite{ishida2017nips}.
Assuming that a single complementary label is provided to each training data, the corresponding loss function considering one-versus-all and pairwise classification was defined, and also the classification risk and error were analyzed.
The classification risk studied by \cite{ishida2019icml} depends on a a more generalized loss function where the classification strategy is not necessarily restricted to be one-versus-all or pairwise classification.

Although choosing the correct class of a pattern involves considerable work for the annotators, the use of complementary labels can alleviate the annotation cost.
However, in this framework, loss of generalities can occur if annotators choose only \textit{one} class as a complementary label.
In other words, it is reasonably assumed that multiple labels can be chosen as classes to which a pattern does not belong.

\cite{Cao2020MultiComplementaryAU,Feng2020LearningFM} analyzed a framework where multiple complementary labels are provided to each training data.
\cite{Cao2020MultiComplementaryAU} formulated the number of complementary labels as a constant value, while \cite{Feng2020LearningFM} formulated that as a random variable.
In these works, the existing framework in \cite{ishida2019icml} where single complementary label is provided to each training data is directly applied.
Therefore, there is no essential difference with the framework of \cite{ishida2019icml} on the structure and properties of the loss functions or the derivation of classification risk and error.

In contrast, this paper focuses on equivalency of ordinary label and complementary label and provides essential generalization of learning framework from single ordinary/comple-mentary label.
The finding that loss functions for one-versus-all and pairwise classification satisfy certain \textit{additivity} and \textit{duality} plays a significant role in the generalization, and the classification risk and error bound with a new expression not found in \cite{ishida2019icml,Cao2020MultiComplementaryAU,Feng2020LearningFM} are derived.
This enables us to bridge ordinary-label learning and complementary-label learning and to understand them from a unified perspective.
To be more specific, the introduced loss functions satisfying \textit{additivity} and \textit{duality} allow a straightforward comparison of the proposed approach and those shown in the existing literature.
The properties also allow our classification risk to be in a more simple form than that of \cite{ishida2017nips, ishida2019icml,Cao2020MultiComplementaryAU, Feng2020LearningFM}.
Further, the derived classification error monotonically decreases corresponding to the number of complementary labels.
This property is also shown in the experiment, which supports that our framework bridges learning from ordinary/complementary labels.

The remaining paper is organized as follows.
Section \ref{sec:related-work} provides a review of the existing work regarding the generalization of supervised learning from ordinary labels.
Section \ref{sec:background} introduces several key formulations pertaining to ordinary-label and complementary-label learning.
Section \ref{sec:formulation} presents the generation probability model of data provided with multiple complementary labels and the loss functions for the learning conducted from such data.
Furthermore, this section describes the natural generalization of the proposed loss function from those for one-versus-all and pairwise classification defined by \cite{ishida2017nips} and the evaluation of the classification risk.
The error bound of the one-versus-all and pairwise classification is derived in Section \ref{sec:errbound}.
Section \ref{sec:experiment} describes the experimental investigation performed considering real-world data to validate the classification error, and Section \ref{sec:conclusion} presents the conclusions.

\section{Related Work}\label{sec:related-work}
This section provides a review of the existing work regarding the generalization of supervised learning from the ordinary labels provided to training data.

\textit{Semi-supervised learning}, in which the classification algorithm is provided with some training data labeled but not necessarily for all, has been investigated \cite{grandvalet2005nips, mann2007icml, chapelle2010mitpress, gang2013pmlr, kipf2016iclr, laine2017iclr}
Although the prediction accuracy of such learning is less than that of fully supervised learning, this technique requires less labeled training data, thereby reducing the annotation cost.

As another type of generalization, there exists a method of learning from training data, in which each data is provided with multiple labels, only one of which specifies the \textit{one} true class.
Such labels are generally known as \textit{partial labels} (or \textit{candidate labels}) \cite{raykar2010jmlr, Jie2010nips, cour2011jmlr, jes2012nips}.
This approach can potentially be applied when the true class of a given training data is not clear or when the label information requires to be kept private.

Framework of learning from labels, namely \textit{complementary labels}, specifying classes that the pattern does not belong to has been also formulated.
Providing a complementary label to a training data is equivalent to considering the other labels as candidate labels.
In other words, a set of candidate labels is a comprehensive concept involving both ordinary and complementary labels as special cases.
\cite{ishida2017nips} provided the complementarily labeled data-generation probability model and classification risk.
In addition, when learning from \textit{noisy labels}, which stochastically represent incorrect information, providing complementary label to prevent the specification of a noisy label as a true label has been noted to be effective \cite{kim2019cvpr}.

However, the problem setting in \cite{ishida2017nips} encounters a limitation in cases in which only one complementary label is provided to each training data, leading to a loss of generality.
To overcome the limitation, \cite{Cao2020MultiComplementaryAU, Feng2020LearningFM} considered a problem setting where each training data is provided to multiple complementary labels. 
The existing framework in \cite{ishida2019icml} is directly applied in these works, and thus new insights into the relationship between ordinary-label learning and complementary-label learning or the properties satisfied by the loss functions are not obtained.

In contrast, we addressed the problem by using an approach different from that of \cite{Cao2020MultiComplementaryAU, Feng2020LearningFM}.
Focusing on equivalency of ordinary label and complementary label, we found that additivity and duality satisfied by loss functions plays a significant role in bridging learning from single ordinary/complementary labels; we then derived classification risk and error bound with a new expression not found in the existing literature.
The structure and properties inherent in the loss functions enable us to naturally relate and discuss learning from ordinary and complementary labels in a unified manner.
\section{Background}\label{sec:background}
This section presents the data-generation probability model and loss functions shown in the existing literature.
All these formulations were on the assumption that a single ordinary label or complementary label is provided.

\subsection{Supervised Learning from Ordinary Label}
In the $K~(\geq2)$ classification problem, supervised learning concerns the efficiency of a classifier $f:\mathcal{X}\rightarrow\mathcal{Y}$, which maps a pattern $x\in\mathcal{X}$ to the class to which it belongs, that is, $y\in\mathcal{Y}$, where $|\mathcal{Y}|=K$.
Given the discriminant function $g_y:\mathcal{X}\rightarrow\mathbb{R}$, which represents the confidence on the $2$ class classification about $y$, the classifier $f$ can be defined as $f(x)=\argmax_{y\in\mathcal{Y}}g_y(x)$.

Given a pair $(x, y)$, in which $x$ is the pattern and $y$ is an ordinary label representing the belonging class of $x$, the prediction by $f$ can be evaluated using a loss function $\ordLoss{}:\mathcal{X}\times\mathcal{Y}\rightarrow\mathbb{R}$.
For example, the loss functions for \textit{one-versus-all} and \textit{pairwise classification}, $\ordLoss[\text{OVA}]{}$ and $\ordLoss[\text{PC}]{}$, respectively, can be defined as \cite{zhang2004jmlr}:
\begin{align}
    \ordLoss[\text{OVA}]{f(x), y}
    &= \loss{g_y(x)}
    + \frac{1}{K-1}
    \sum\limits_{
    y' \neq y
    }\loss{-g_{y'}(x)}
    \label{eq:ova-ord-loss}\\
    \ordLoss[\text{PC}]{f(x), y}
    &= \sum\limits_{
    y' \neq y}
    \loss{g_y(x)-g_{y'}(x)}
    \label{eq:pc-ord-loss}
\end{align}
Note that the function $\loss{}:\mathbb{R}\rightarrow\mathbb{R}_+$ monotonically decreases corresponding to the input.

The 0-1 loss, $\loss[\text{0-1}]{z}:=I(z\leq0)$, is a standard type of function $\loss{}$, where $I$ is an indicator function.
The 0-1 loss is unsuitable for loss optimization, as it is undifferentiable at $z=0$, and its gradient is always $0$ for all other inputs.
Consequently, the 0-1 loss is usually surrogated by other functions, such as the sigmoid and ramp losses, all of which satisfy the following condition:
\begin{align}
    \loss{z} + \loss{-z} &= a
    \label{eq:odd-function}
\end{align}
where $a\in\mathbb{R}$ is constant.
In the remaining study, we assume $\loss{}$ to satisfy \eqref{eq:odd-function} but not its differentiability.

In addition, we assume that $(x, y)$, involving the pattern $x$ and its belonging class $y$, is generated from a probability distribution $\OrdP{x, y}$.
The classification risk $R(f)$ of a classifier $f$ can be defined as follows:
\begin{align}
    R(f)
    &=\Expected[\OrdP{x, y}]{\ordLoss{f(x), y}}
    \label{eq:ord-risk}
\end{align}
where $\Expected[\OrdP{x, y}]{}$ represents the expectation with respect to $\OrdP{x, y}$.

\subsection{Supervised Learning from Complementary Label}
\cite{ishida2017nips} discussed the formulation considering a label specifying a class that a pattern does not belong to.
Given a pair $(x, \overline{y})$ of a pattern $x\in\mathcal{X}$ and a class $\overline{y}\in\mathcal{Y}$ to which the pattern does not belong, the prediction by $f:\mathcal{X}\rightarrow\mathcal{Y}$ can be evaluated by the loss function $\compLoss{}: \mathcal{X}\times\mathcal{Y}\rightarrow\mathbb{R}$.
\cite{ishida2017nips} defined the loss functions for one-versus-all and pairwise classification as follows:
\begin{align}
    \compLoss[\text{OVA}]{f(x), \overline{y}}
    &= \frac{1}{K-1}\sum\limits_{
    y \neq \overline{y}
    }\loss{g_y(x)}
    + \loss{-g_{\overline{y}}(x)}
    \label{eq:ova-comp-loss}\\
    \compLoss[\text{PC}]{f(x), \overline{y}}
    &= \sum\limits_{
    y \neq \overline{y}}
    \loss{g_y(x) - g_{\overline{y}}(x)}
    \label{eq:pc-comp-loss}
\end{align}
Compared to \eqref{eq:ova-ord-loss} and \eqref{eq:pc-ord-loss}, the loss functions defined in \eqref{eq:ova-comp-loss} and \eqref{eq:pc-comp-loss} are in natural forms to deal with the training data provided a label $\overline{y}$.
The extended form of the loss functions in \eqref{eq:ova-comp-loss} and \eqref{eq:pc-comp-loss} is as follows \cite{ishida2019icml}: 
\begin{align}
    \compLoss{f(x), \overline{y}}
    &= -(K-1)\ordLoss{f(x),\overline{y}}
    + \sum\limits_{y\in\mathcal{Y}}
    \ordLoss{f(x),y}
    \label{eq:icml-risk}
\end{align}
Considering \eqref{eq:odd-function}, the function in \eqref{eq:icml-risk} involves both one-versus-all and pairwise classification losses.

Given a pattern $x$ and its complementary label $\overline{y}\in\mathcal{Y}$, the generation probability $\CompP{\overline{y}|x}$ of the labeled data is as follows:
\begin{align}
    \CompP{\overline{y}|x}
    &= \frac{1}{K-1}
    \sum\limits_{
    y \neq \overline{y}
    }\OrdP{y|x}
    \label{eq:nips-eq5}
\end{align}
Here, $\CompP{\overline{y}|x}$ is proportional to the sum of probabilities regarding all the other labels.

Subsequently, $(x, \overline{y})$ is assumed to be stochastically generated from the distribution $\CompP{x,\overline{y}}=\CompP{\overline{y}|x}P(x)$.
The $R(f)$ defined in \eqref{eq:ord-risk} can be expressed as: 
\begin{align}
    R(f)
    &= (K-1)\Expected[\CompP{x, \overline{y}}]
    {\compLoss{f(x), \overline{y}}}
    - \overline{m}_{1} + m_{2}
    \label{eq:comp-risk}
\end{align}
such that $\overline{m}_{1}$ and $m_{2}$ are constants, where $\overline{m}_{1}:=\sum_{y\in\mathcal{Y}}\compLoss{f(x), y}$ and $m_{2}:=\ordLoss{f(x), y}+\compLoss{f(x), y}$.

\section{Formulation}\label{sec:formulation}
The formulation in \cite{ishida2017nips} assumes that a single complementary label is provided to the training data. 
We generalize this formulation, assuming that $M~(1\leq M \leq K-1)$ complementary labels are provided to each training data.
In this section, we generalize loss functions for one-versus-all and pairwise classification focusing on the fact that those functions satisfy \textit{additivity} and \textit{duality}.
Further, we derive classification risk assuming the data-generation probability model defined in \cite{ishida2017nips}, and then show that additivity and duality play a significant role in the analysis.

The labels specifying the incorrect classes are equally informative as the remaining labels that specify the classes to which a pattern likely belongs.
The latter labels are termed as \textit{candidate labels} and the set of these labels is defined as $Y\in\PowerSet{\mathcal{Y}}$, where $\PowerSet{\mathcal{Y}}$ denotes the set of all the $N$-size subsets of $\mathcal{Y}$.
From now on the formulation in this paper is constructed based on the fact that providing $M~(=K-N)$ complementary labels to a training data is equivalent to providing $N~(=|Y|=K-M)$ candidate labels.

\subsection{Generalization of Loss Function}\label{subsec:loss-function}
For one-versus-all and pairwise classification from a set of multiple candidate labels $Y$, the corresponding loss functions $\CandLoss[\text{OVA}]{}$ and $\CandLoss[\text{PC}]{}$ can be defined as:
\begin{align}
    &\CandLoss[\text{OVA}]{f(x), Y}
    = \frac{K-N}{K-1}\sum\limits_{y\in Y}\loss{g_y(x)}
    + \frac{N}{K-1}\sum\limits_{\overline{y}\notin Y}\loss{-g_{\overline{y}}(x)}
    \label{eq:ova-Cand-loss}\\
    &\CandLoss[\text{PC}]{f(x), Y}
    = \sum\limits_{y\in Y}\sum\limits_{\overline{y}\notin Y}
    \loss{g_y(x) - g_{\overline{y}}(x)}
    \label{eq:pc-Cand-loss}
\end{align}
Note that \eqref{eq:ova-Cand-loss} and \eqref{eq:pc-Cand-loss} are the same as \eqref{eq:ova-ord-loss} and \eqref{eq:pc-ord-loss}, respectively, if $N=1$.
Similarly, the newly defined equations are the same as \eqref{eq:ova-comp-loss} and \eqref{eq:pc-comp-loss} if $N=K-1$.

As a generalization of \eqref{eq:ova-Cand-loss} and \eqref{eq:pc-Cand-loss}, we introduce a loss function  $\CandLoss{}:\mathcal{X}\times\mathcal{Y}\rightarrow\mathbb{R}$, which is defined as the \textit{additive form} of loss functions for ordinary-label learning:
\begin{align}
    \CandLoss{f(x), Y}
    &= \xi_1\sum\limits_{y\in Y}
    \ordLoss{f(x), y}
    + \xi_2
    \label{eq:CandLoss}
\end{align}
where $\xi_1\in\mathbb{R}_+$ and $\xi_2\in\mathbb{R}$ are constants.
Assuming $\ordLoss{f(x), y}=\ordLoss[\text{OVA}]{f(x), y}$, $\xi_1=1$, and $\xi_2=-\frac{a N(N-1)}{K-1}$, \eqref{eq:CandLoss} becomes the same expression as \eqref{eq:ova-Cand-loss} if $\loss{}$ satisfies \eqref{eq:odd-function}.
Under the same condition on $\loss{}$, by assuming $\ordLoss{f(x), y}=\ordLoss[\text{PC}]{f(x), y}$, $\xi_1=1$, and $\xi_2=-a\cdot\comb{N}{2}$, \eqref{eq:CandLoss} and \eqref{eq:pc-Cand-loss} have the same expression (see supplementary materials for complete proofs).

Similar to the loss functions defined in \eqref{eq:icml-risk}, $\CandLoss{}$ in \eqref{eq:CandLoss} is defined from the ordinary loss function $\ordLoss{}$.
In contrast to the loss function defined in \eqref{eq:icml-risk}, which can be applied only if a single complementary label is provided, the loss function pertaining to \eqref{eq:CandLoss} can be applied for any number of complementary (or candidate) labels.
Consequently, $\CandLoss{}$ in \eqref{eq:CandLoss} can be considered as a generalized form of \eqref{eq:ova-comp-loss} and \eqref{eq:pc-comp-loss}.

Furthermore, for any $x$, let the loss function $\ordLoss{}$ satisfy the following condition:
\begin{align*}
    \sum_{y \in \mathcal{Y}}\ordLoss{f(x), y} = m_{1}
\end{align*}
where $m_{1} \in \mathbb{R}$ is a constant.
Both $\ordLoss[\text{OVA}]{}$ and $\ordLoss[\text{PC}]{}$ satisfy this condition, where $m_{1} = a K$ and $m_{1} = a \cdot \comb{K}{2}$, respectively.
For any $x$ and $y$, we then define a function $\compLoss{}$, which satisfies the following condition:
\begin{align*}
    \ordLoss{f(x), y}+\compLoss{f(x), y} = m_{2}
\end{align*}
where $m_{2} \in \mathbb{R}$ is a constant.
This condition is satisfied by $\ordLoss[\text{OVA}]{}$ and $\compLoss[\text{OVA}]{}$ and also by $\ordLoss[\text{PC}]{}$ and $\compLoss[\text{PC}]{}$, where $m_{2} = 2 a$ and $a(K-1)$ respectively.

Now, the following equation holds:
\begin{align}
    \CandLoss{f(x), Y}
    &= \xi_1\sum\limits_{\overline{y}\in \overline{Y}}
    \compLoss{f(x), \overline{y}}
    + \overline{\xi}_2
    \label{eq:CompLoss}
\end{align}
where $\overline{Y}\cup Y=\mathcal{Y}$ and $\overline{Y}\cap Y=\emptyset$.
Note that $\overline{\xi}_2 = \xi_{1} m_{1} + \xi_{2} - \xi_{1} m_{2}(K-N)$, and it can be expressed as $\xi_{2} = \overline{\xi}_{2}$ if $m_{1} = m_{2} = 0$.
\eqref{eq:CandLoss} and \eqref{eq:CompLoss} express the \textit{duality} of the loss function $\CandLoss{}$.

To the best of our knowledge, we are the first to reveal additivity and duality commonly found in loss functions for one-versus-all and pairwise classification.
This property is critical to study the classification risk and error, and it has not been discussed in any of the existing studies \cite{ishida2019icml, Cao2020MultiComplementaryAU}.

\subsection{Assumption of Generation Probability}\label{subsec:gen-probability}

As stated in Section \ref{sec:background}, \eqref{eq:nips-eq5} represents the probability of $\overline{y}$ not being a true label.
This aspect can be interpreted in the context of a candidate label; $\CompP{\overline{y}|x}$ represents the probability with which a true label is included in $Y$.
In contrast to \eqref{eq:nips-eq5}, which can only be applied when $M=1$ (or $N=|\mathcal{Y}|-1=K-1$), our generalized data-generation probability model is defined as:
\begin{align}
    \CandP{Y|x}
    &= \frac{1}{\comb{K-1}{N-1}}
    \sum\limits_{y\in Y}\OrdP{y|x}
    \label{eq:CandP}
\end{align}
where $\CandP{Y|x}$ represents the probability of the true label being included in $Y$.
In the remaining work, we assume that $Y$ is generated from $\CandP{Y|x}$ independently; \eqref{eq:CandP} is obviously the same as \eqref{eq:nips-eq5} if $N=K-1$.

This model represents the situation where the annotator is constrained to provide $N$ labels consistently to any pattern $x$.
This situation can be seen as synthetic; that is, the labelling is constrained to be independent from the annotator's belief.
In fact, this model does not consider the bias in labeling by the annotator because the information in $y$ that satisfies $y\in Y$ is treated uniformly.
However, this property does not always limit reality of \eqref{eq:CandP}.
In the following explanation we show that the synthetic property of \eqref{eq:CandP} can be reasonably interpreted in a context of privacy preserving.

Given a data provided with a set of candidate labels which is generated from \eqref{eq:CandP}, consider the information about the ground-truth label obtained from the given data.
We define a function $Q(\alpha|x) = \Pr\{y_{o} = \alpha|x\}$ on $\mathcal{Y}$, which represents the degree of confidence where the one given the data believes that the ground-truth label $y_o$ for the pattern $x$ is $\alpha\in\mathcal{Y}$.
Now we assume that a set of candidate labels $Y$ is provided to the pattern $x$ by the annotator.
Further, we assume that the one given the data does not have any information about the ground-truth label $y_o$, except the given set of candidate labels $Y$.
That is, we assume that the degree of confidence, where the ground-truth label $y_o$ for the pattern $x$ is $\alpha\in\mathcal{Y}$, depends on whether $\alpha \in Y$ or $\alpha \not\in Y$, thus $Q(\alpha|x, \alpha \in Y)=1/N$ and $Q(\alpha|x, \alpha \not\in Y)=0$ hold.
Here, the following equation holds (see supplementary materials for complete proofs).
\begin{align}
Q(\alpha|x) = \beta P(\alpha|x) + (1-\beta)\frac{1}{K},\label{eq:privacy}
\end{align}
where
\begin{align*}
    \beta = \frac{K-N}{N(K-1)}.
\end{align*}

According to \eqref{eq:privacy}, the degree of confidence $Q(\alpha|x)$ is a mixture of the distribution $P(\alpha|x)$ and the uniform distribution $1/K$; $Q(\alpha|x)$ represents the information about labels that is available to the one given the data, while $P(\alpha|x)$ represents the information about labels originally owned by the annotator.
This equation indicates that receiving a data provided with a set of $N$ candidate labels generated from \eqref{eq:CandP} is equivalent to receiving a data provided with a ordinary label to which random noise is added.
The level of random noise $(1-\beta)$ increases as the number of candidate labels $N$ increases.
That is, privacy for the generation probability of the label for a given pattern, $P(y|x)$, can be preserved by synthetically generating candidate labels according to \eqref{eq:CandP}.

Note that \cite{Cao2020MultiComplementaryAU, Feng2020LearningFM} defined data-generation probability model in the same expression assuming multiple complementary labels are provided.
Here, \cite{Cao2020MultiComplementaryAU} assumes the number of labels as a constant value while \cite{Feng2020LearningFM} assumes that as a stochastic variable.
Those probability models are mathematically equivalent to \eqref{eq:CandP}, which can be easily shown; assuming $\mathcal{Y} = Y\cup\overline{Y}$ and $\emptyset = Y\cap\overline{Y}$, the number of patterns choosing $N$ candidate labels from the set of all the labels $\mathcal{Y}$ is equal to the number of patterns choosing $K-N$ complementary labels.
Despite the fact, both the probability models defined in \cite{Cao2020MultiComplementaryAU, Feng2020LearningFM} are not in the expression clarifying the relationship between ordinary label and complementary label, as those are straightforward extension of the probability model defined in \cite{ishida2017nips} from a perspective of increasing the number of complementary labels.
Further, to the best of our knowledge, this paper is the first to explicitly provide an interpretation of the information obtained from the labels generated by \eqref{eq:CandP}; \eqref{eq:privacy} clarifies that the information appears to be obtained from a mixture of the true distribution and uniform distribution.
The goal of this paper is to discuss recent studies in complementary labels in a unified manner based on a concept of candidate labels.
Thus, we use the expression form in \eqref{eq:CandP} to naturally relate ordinary-label learning and complementary-label learning. 

\subsection{Classification Risk with Multiple Candidate Labels}\label{subsec:simple-risk}
If we define loss function $\CandLoss{}$ as in \eqref{eq:CandLoss}, the following theorem holds for any type of $\ordLoss{}$.
Similar to \eqref{eq:ord-risk} and \eqref{eq:comp-risk}, the theorem allows the expression of the classification risk from multiple complementary labels in terms of the expectation of loss. 
The complete proof is available in the supplementary materials.

\begin{theorem}\label{thm:risk}
Given a pattern $x$, a set of candidate labels $Y$, and loss function $\CandLoss{}$ defined by \eqref{eq:CandLoss}, the classification risk $R(f)$ can be expressed as follows:
\begin{align}
    R(f)
    &= \frac{K-1}{
    \xi_1 \left(K-N\right)}
    \Expected[\CandP{x, Y}]{\CandLoss{f(x), Y}}+ C
    \label{eq:cand-risk}
\end{align}
where
\begin{align}
    C
    &= - \frac{N-1}{K-N}\sum\limits_{y\in \mathcal{Y}}\ordLoss{f(x), y}
    - \frac{\xi_2\left(K-1\right)}{
    \xi_1\left(K-N\right)}
    \label{eq:cand-risk-const}
\end{align}
\end{theorem}

Some conditions can be used to simplify the expression in \eqref{eq:cand-risk} for one-versus-all and pairwise classification, without losing generalities.
Note that redefining \eqref{eq:odd-function} as $\loss{z}:=\loss{z}-\loss{0}=\loss{z}-\frac{a}{2}$ does not affect the loss minimization when learning.
If we shift $\loss{}$ to satisfy $a=0$, $m_{1}=0$ holds for both $\ordLoss[\text{OVA}]{}$ and $\ordLoss[\text{PC}]{}$; therefore, the first term in \eqref{eq:cand-risk-const} can be eliminated.

Similarly, $\xi_1$ and $\xi_2$ in \eqref{eq:CandLoss} do not affect the loss minimization.
Considering $\xi_1=1$ and $\xi_2=0$, we can simplify \eqref{eq:CandLoss} as follows without loss of generalities:
\begin{align}
    \CandLoss{f(x), Y}
    = \sum\limits_{y\in Y}
    \ordLoss{f(x), y}
    \label{eq:simple-CandLoss}
\end{align}
Consequently, by assuming that $a=0$, $\xi_1=1$, and $\xi_2=0$, the following expression holds for the classification risk.
\begin{align}
    R(f)
    &= \frac{K-1}{K-N}
    \Expected[\CandP{x, Y}]{\CandLoss{f(x), Y}}
    \label{eq:simple-cand-risk}
\end{align}

\section{Statistical Analysis}\label{sec:errbound}
This section discusses the error bound for one-versus-all and pairwise classification.
For simplicity, we assume that $\loss{}$ satisfies \eqref{eq:odd-function} with $a=0$, and the conditions $\inf_{z} \loss{z} = -1/2$ and $\sup_{z} \loss{z} = 1/2$.
Further, we assume that $\loss{}$ is Lipschitz continuous.
In addition, we assume $\xi_1=1$ and $\xi_2=0$.
For the rest of this paper, we define $\CandLoss{}$ and $R(f)$ by using \eqref{eq:simple-CandLoss} and \eqref{eq:simple-cand-risk}, respectively.

\subsection{Notations}\label{subsec:notation}
Let us consider that a set of $n$ training data $\mathcal{S}=\left\{(x_i, Y_i)\right\}_{i=1}^{n}$ is given, and each training data is generated with a probability of $\CandP{x, Y}$ independently.
Based on \eqref{eq:simple-cand-risk}, the empirical classification risk $\hat{R}(f)$ for the set $\mathcal{S}$ is:
\begin{align*}
    \hat{R}(f)
    &= \frac{K-1}{n \left(K-N\right)}
    \sum^n_{i=1}{\CandLoss{f(x_i), Y_i}}
\end{align*}
We define the ideal classifier that minimizes the generalization error (\textit{Bayes classifier}) and the empirically ideal classifier as $f^* := \argmin_{f}R(f)$ and $\hat{f} := \argmin_{f}\hat{R}(f)$, respectively.
We define the classification error $\err$ for the classifier $\hat{f}$ as follows. 
\begin{align}
    \err
    &= R(\hat{f}) - R(f^*)
    \label{eq:error}
\end{align}

In the literature, the \textit{Rademacher complexity} for a set of discriminant functions $\mathcal{G}$ over the input space $\mathcal{X}$ is usually defined as follows:
\begin{align*}
    \Rad{\mathcal{G}}
    &= \Expected[\mathcal{S}][\mathcal{\sigma}]{
    \sup_{g\in\mathcal{G}}\frac{1}{n}
    \sum_{i=1}^n\sigma_i g(x_i)}
\end{align*}
where $\sigma=\left\{\seq{\sigma}{n}\right\}$ is a set of independent stochastic variables, which take one value of $\left\{-1, +1\right\}$ with the same probability.
In addition, $\Expected[\mathcal{S}]{}$ and $\Expected[\mathcal{\sigma}]{}$ represent the expectation for each element of $\mathcal{S}$ and $\mathcal{\sigma}$, respectively.

\subsection{Evaluation of Error Bound}\label{subsec:error-bound}
The following lemmas are introduced to derive the classification error bound.

\begin{lemma}\label{lem:candloss-norm}
We express the supremum of the difference in loss $\|\CandLoss{}\|_{\infty}$ in accordance with the change in a set of candidate labels, i.e., given any $Y, Y'\in\PowerSet{\mathcal{Y}}$,
\begin{align*}
    &\|\CandLoss[]{}\|_{\infty}
    =\sup\limits_{\seq{g}{K}\in\mathcal{G}}\left(
    \sum\limits_{y\in Y}\ordLoss[]{f(x), y}
    - \sum\limits_{y'\in Y'}\ordLoss[]{f(x), y'}\right)
\end{align*}
Then, the following holds for one-versus-all classification.
\begin{align*}
    \|\CandLoss[\text{OVA}]{}\|_{\infty}
    = 
    \begin{cases}
    \displaystyle\frac{K N}{K-1}, & \text{if $N\le \displaystyle\frac{K}{2}$}\\[.5em]
    \displaystyle\frac{K (K-N)}{K-1}, & \text{otherwise}\\
    \end{cases}
\end{align*}
Similarly, for pairwise classification,
\begin{align*}
    \|\CandLoss[\text{PC}]{}\|_{\infty}
    &= N(K-N)
\end{align*}
\end{lemma}

\begin{lemma}\label{lem:rademacher-bound}
Define a function set $\mathcal{H}_{{\text{OVA}}}$, $\mathcal{H}_{{\text{PC}}}$ as follows:
\begin{align*}
    \mathcal{H}_{{\text{OVA}}}
    &=\left\{(x, Y)\mapsto\CandLoss[{\text{OVA}}]{f(x), Y}\mid
    \seq{g}{K}\in\mathcal{G}\right\}\\
    \mathcal{H}_{{\text{PC}}}
    &=\left\{(x, Y)\mapsto\CandLoss[{\text{PC}}]{f(x), Y}\mid
    \seq{g}{K}\in\mathcal{G}\right\}
\end{align*}
Then, if $\loss{}$ is $\Lip~(\geq0)$ Lipschitz continuous, the following holds for $\mathcal{H}_{{\text{OVA}}}$.
\begin{align*}
    \Rad{\mathcal{H}_{{\text{OVA}}}}\leq
    \begin{cases}
    \displaystyle\frac{K (K+N)}{K-1}\Lip\Rad{\mathcal{G}}, & \text{if $N\le \displaystyle\frac{K}{2}$}\\[.5em]
    \displaystyle\frac{K (2 K-N)}{K-1}\Lip\Rad{\mathcal{G}}, & \text{otherwise}\\
    \end{cases}
\end{align*}
Similarly, for $\mathcal{H}_{{\text{PC}}}$,
\begin{align*}
    \Rad{\mathcal{H}_{{\text{PC}}}}\leq2 K (K-1)\Lip\Rad{\mathcal{G}}
\end{align*}
\end{lemma}

Based on these lemmas, the error bounds for the one-versus-all and pairwise classification can be defined as follows.
The complete proofs are provided in the supplementary materials.
\begin{theorem}\label{thm:errbound}
Assume that function $\loss{}$ satisfies the stated condition in the beginning of this section, and is $\Lip$ Lipschitz continuous.
Then, for any $\delta~(>0)$, the following equation for one-versus-all classification holds with a probability of at least $1-\delta$.
\begin{align}
    \err&\leq
    \begin{cases}
    \displaystyle\frac{4 K (K+N)}{K-N}
    \Lip\Rad{\mathcal{G}}
    + \displaystyle\frac{K N}{K-N}
    \sqrt{\frac{2\ln(2/\delta)}{n}},& \text{if}~ N \leq \displaystyle\frac{K}{2} \\
    \displaystyle\frac{4 K (2 K - N)}{K - N}\Lip\Rad{\mathcal{G}}
    + K\sqrt{\displaystyle\frac{2\ln(2/\delta)}{n}}, & \text{otherwise}
    \end{cases}
    \label{eq:ova-error}
\end{align}
Similarly, for pairwise classification, the following equation holds with a probability of at least $1-\delta$.
\begin{align}
    \err
    &\leq \frac{8 K (K-1)^{2}}{K - N}\Lip\Rad{\mathcal{G}}
    +  N(K-1)\sqrt{\frac{2\ln{(2/\delta)}}{n}}
    \label{eq:pc-error}
\end{align}
\end{theorem}

Note that the upper-bound of $\err$ increases monotonically corresponding to $N$ in accordance with \eqref{eq:ova-error} and \eqref{eq:pc-error}.
This aspect is in agreement with the fact that a decrease in the number of candidate labels leads to less ambiguous supervision of the training data. 
The upper-bound in \eqref{eq:ova-error} breaks subject to $N\leq K/2$; this property is due to the equivalency satisfied by ordinary labels and complementary labels, as discussed in Section \ref{subsec:gen-probability}.
Note that taking the expectation of the error bounds in the theorem \ref{thm:errbound} assuming that $N$ as a random variable, the derived formulations specify the error bounds under the problem setting where the number of candidate labels provided to each training data stochastically fluctuates.

\section{Experiment}\label{sec:experiment}
This section describes the evaluation of the accuracy of one-versus-all and pairwise classification and the validation of the formulation discussed in Section \ref{subsec:error-bound}. Understanding the exact behavior of the classification error only from the derived equations in Theorem~\ref{thm:errbound} is difficult.
Therefore, we attempt to quantitively discuss the error in a real-world classification problem.
The source code for the described experiment is available online\footnote{\url{https://github.com/YasuhiroKatsura/ord-comp}}.

\begin{figure*}[t]\centering
    \subfigure[$K=10$ classification]{
		\includegraphics[keepaspectratio, width=1.00\linewidth]{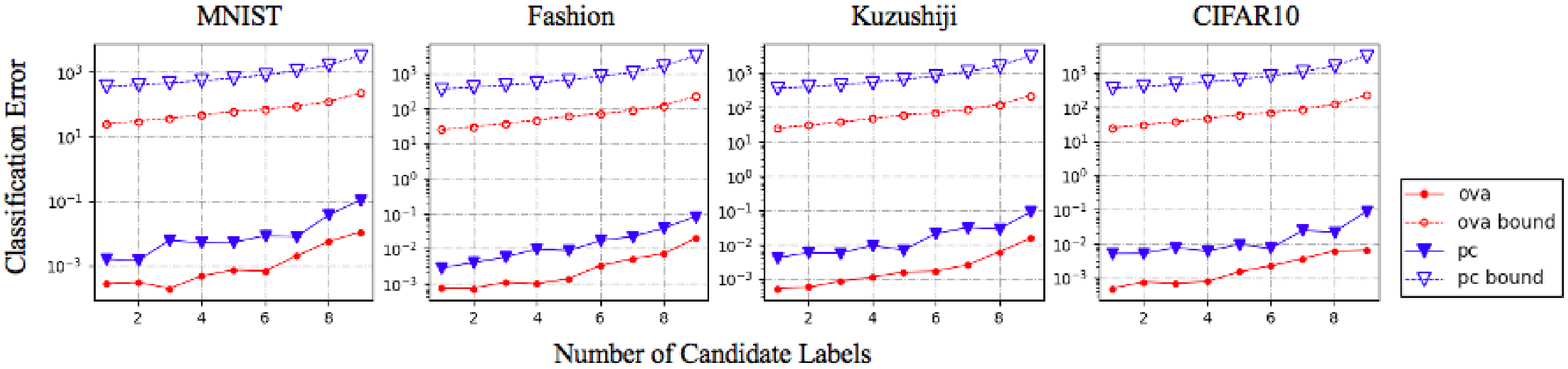}}

	\subfigure[$K=5$ classification]{
		\includegraphics[keepaspectratio, width=1.00\linewidth]{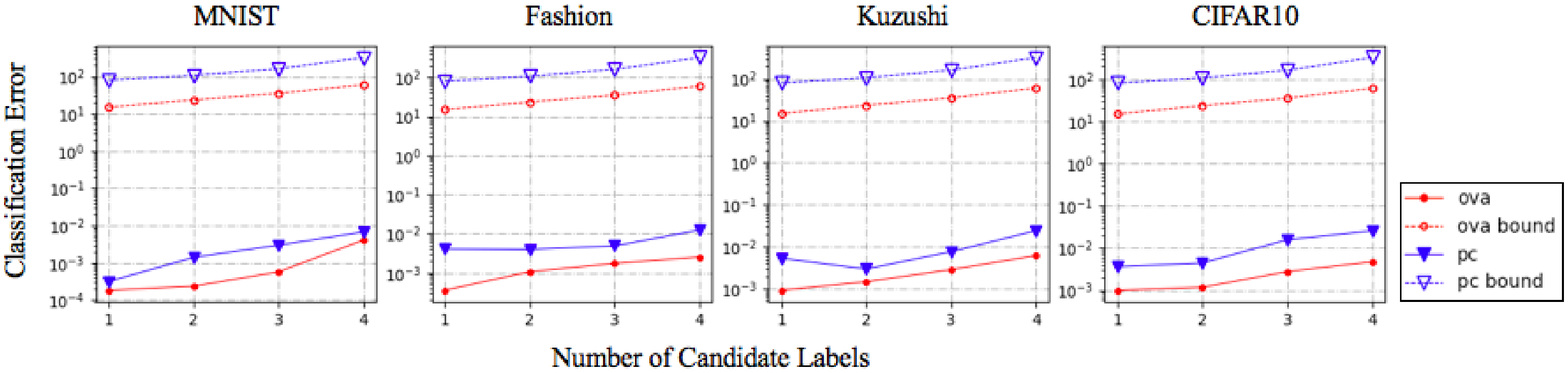}}
    \caption{
    Error for $10$ classification for different numbers of complementary labels $N$.
    The red and blue closed plots represent the experimental results for the one-versus-all and pairwise classification, respectively.
    The red and blue open plots represent the theoretical error bounds for the one-versus-all and pairwise classification, respectively.}\label{fig:error}
\end{figure*}

\subsection{Dataset Generation}
The generation probability model of the training data provided with multiple candidate labels is as defined in \eqref{eq:CandP}.
We generate the experimental dataset according to this definition, and the annotation of the training data is performed as follows.
First, we prepare pretrained $K$ class classifiers, namely, \textit{annotators}, for the ordinarily labeled training data.
The annotators are multi-layer perceptrons (MLP) with softmax as the output layer activation function. The annotators are trained on the MNIST\footnote{\url{http://yann.lecun.com/exdb/mnist/}}, Fashion-MNIST\footnote{\url{https://github.com/zalandoresearch/fashion-mnist}}, Kuzushiji-MNIST\footnote{\url{https://github.com/rois-codh/kmnist}}, and CIFAR-10\footnote{\url{https://www.cs.toronto.edu/~kriz/cifar.html}} datasets.
Table \ref{tab:dataset} summarizes the datasets.
Because all these cases correspond to the $10$ class ($y \in \left\{ 0, \cdots, 9\right\}$) classification problem, the data belonging to $y = 5, \cdots, 9$ is eliminated when we assume $K=5$ in the following experiment.
Because the discriminant functions of the annotators $\seq{g}{K}$ are normalized using softmax, each of the functions can be treated as having a data generation probability $P(y|x)$.
Therefore, assuming that the number of candidate labels is $N$, we generate $Y$ in accordance with $\CandP{f(x), Y}$, calculated using the discriminant functions.
Similarly, for the test dataset, the true label $y\in\mathcal{Y}$ is provided in accordance with $\seq{g}{K}$ of the annotators.

In the experiment, we pretrained the annotators for $K=10$ and $K=5$.
When $K=10$, the classification accuracies of the annotators for the MNIST, Fashion-MNIST, Kuzushiji-MNIST, and CIFAR-10 datasets are $99.17\%$, $92.67\%$, $95.1\%$, and $81.62\%$, respectively.
For $K=5$, the corresponding accuracies are $99.90\%$, $95.16\%$, $96.56\%$, and $85.60\%$.
As all the accuracies are reasonably high, the generation probability $P(y|x)$ for each class $y\in\mathcal{Y}$ is not uniform.
As discussed in section \ref{subsec:data-generation}, providing a set of candidate labels generated from \eqref{eq:CandP} is equivalent to providing a single ordinary label generated from the true probability distribution $P(y|x)$ with random noise from uniform distribution.
Therefore, generating synthetic data as in the experiment is natural in the context of preserving privacy of training data.

Here, we show the data augmentation algorithm used in the experiment in \algorithmref{alg:data-augmentation}.
\algorithmref{alg:data-augmentation} requires $n$ patterns $\{x_i\}_{i=1}^{n}$ as inputs and returns $\{x_i, Y_i\}_{i=1}^{n}$ by providing a set of candidate labels to each of input patterns.
Further, we define \texttt{Annotator} as a pretrained multi-layer perceptron which returns degree of classification confidence for each patterm $x_i$ and \texttt{ChooseCandidates} as a stochastic function which returns a set of candidate labels based on \eqref{eq:CandP}.
\begin{table}[t]\centering
\caption{
Datasets used in the experiment.
}\scalebox{1.0}{
  \begin{tabular}{c|c c c c} \hline
     & MNIST & Fashion & Kuzushiji & CIFAR-10 \\ \hline
     Number of classes & 10 & 10 & 10 & 10 \\
     Number of dimensions & 28$\times$28 & 28$\times$28 & 28$\times$28 & 3$\times$32$\times$32 \\
      & (gray) & (gray) & (gray) & (RGB) \\
     Number of data & 60,000 & 60,000 & 60,000 & 60,000 \\ \hline
  \end{tabular}
  }\label{tab:dataset}
\end{table}
\begin{algorithm2e}[t]
\caption{Data Augmentation}
\label{alg:data-augmentation}
\KwIn{the number of the total classes $K$, the number of candidate labels $N$, set of $n$ patterns $\{x_i\}_{i=1}^{n}$}
\KwOut{set of training data with candidate labels $\{x_i, Y_i\}_{i=1}^{n}$}
\SetKwFunction{FnAnnotator}{Annotator}
    \SetKwProg{Fn}{Function}{:}{}
    \Fn{\FnAnnotator{$x$}}{
    Calculate classification confidence $p_i\in[0, 1]$ for each class ($i=1, \cdots, K$).\\
    \KwRet $(p_1, \cdots, p_K)$ 
}
\SetKwFunction{FnChooseCandidates}{ChooseCandidates}
    \SetKwProg{Fn}{Function}{:}{}
    \Fn{\FnChooseCandidates{$\bm{p}$}}{
    Calculate generation probability by \eqref{eq:CandP} for all the possible combinations of $N$ candidate labels.\\
    Then choose a set of candidate labels from the combinations randomly.\\
    \KwRet $(y_1, \cdots, y_N)$
}
\For{$i\leftarrow 1$ \KwTo $n$}{
$\bm{p} \leftarrow$ \FnAnnotator{$x_i$}\\
$Y_i \leftarrow$ \FnChooseCandidates{$\bm{p}$}
}
\KwRet $\{x_i, Y_i\}_{i=1}^{n}$
\end{algorithm2e}

\begin{table*}[t]\centering
\caption{
  Experimental classification accuracies for $10$ and $5$ class classification (\%).
  The experiments were performed $5$ times for each case; the mean accuracy and standard deviation are presented by the upper and lower values, respectively.
  The highest accuracy is boldfaced.
}\scalebox{0.62}{
  \begin{tabular}{c|c c c c c c c c c|c c c c} \hline
    &&&&& $K=10$ &&&&&& $K=5$ \\
    N & 1 & 2  & 3 & 4 & 5 & 6 & 7 & 8 & 9 & 1 & 2 & 3 & 4 \\ \hline
    MNIST &\bf{92.28} &88.22 &89.44 &86.86 &80.72 &85.99 &77.28 &72.27 &62.13
    &\bf{98.06} &97.77 &97.06 &95.40  \\
    (OVA) &($\pm$3.45) &($\pm$6.97) &($\pm$4.23) &($\pm$4.16) &($\pm$6.05) &($\pm$4.82) &($\pm$5.02) &($\pm$6.42) &($\pm$11.03) 
    &($\pm$0.12) &($\pm$0.18) &($\pm$0.20) &($\pm$0.18)  \\
    MNIST &\bf{92.38} &91.58 &85.94 &87.38 &84.93 &83.93 &84.79 &71.73 &67.07 
    &\bf{98.12} &97.61 &97.12 &95.80  \\
    (PC) &($\pm$3.42) &($\pm$3.57) &($\pm$6.57) &($\pm$4.26) &($\pm$6.53) &($\pm$3.57) &($\pm$3.94) &($\pm$4.63) &($\pm$4.66) 
    &($\pm$0.13) &($\pm$0.14) &($\pm$0.16) &($\pm$0.29)  \\ \hline
    Fashion &\bf{81.07} &80.66 &80.43 &80.62 &79.64 &78.99 &75.10 &74.89 &71.41 
    &\bf{86.42} &85.63 &84.94 &83.37  \\
    (OVA) &($\pm$0.09) &($\pm$0.32) &($\pm$1.05) &($\pm$1.15) &($\pm$1.03) &($\pm$1.49) &($\pm$5.58) &($\pm$2.61) &($\pm$2.59) 
    &($\pm$0.26) &($\pm$0.83) &($\pm$0.41) &($\pm$0.37)\\
    Fashion &\bf{81.47} &81.51 &80.41 &81.23 &79.24 &79.88 &78.45 &75.64 &73.11 
    &\bf{86.52} &85.74 &84.74 &83.14 \\
    (PC) &($\pm$0.16) &($\pm$0.95) &($\pm$0.16) &($\pm$1.32) &($\pm$0.38) &($\pm$1.29) &($\pm$0.68) &($\pm$3.79) &($\pm$2.30) 
    &($\pm$0.33) &($\pm$0.44) &($\pm$0.37) &($\pm$0.41)  \\ \hline
    Kuzushi &\bf{64.19} &60.83 &58.96 &57.45 &52.20 &51.28 &49.43 &41.27 &37.09 
    &\bf{80.49} &79.02 &75.78 &70.55  \\
    (OVA) &($\pm$2.45) &($\pm$2.33) &($\pm$5.90) &($\pm$3.55) &($\pm$2.96) &($\pm$1.60) &($\pm$5.65) &($\pm$3.46) &($\pm$3.76) 
    &($\pm$0.33) &($\pm$0.61) &($\pm$0.25) &($\pm$1.08)  \\
    Kuzushi &\bf{66.15} &63.67 &58.53 &57.20 &55.96 &54.28 &49.74 &43.62 &34.82 
    &\bf{80.97} &78.78 &76.53 &70.23  \\
    (PC) &($\pm$0.48) &($\pm$2.06) &($\pm$5.82) &($\pm$4.09) &($\pm$3.44) &($\pm$1.95) &($\pm$4.02) &($\pm$0.66) &($\pm$3.07) 
    &($\pm$0.22) &($\pm$0.47) &($\pm$0.65) &($\pm$1.79)  \\ \hline
    CIFAR-10 &\bf{57.17} &53.44 &51.16 &45.88 &37.95 &34.85 &25.54 &27.43 &24.18 
    &\bf{65.72} &63.76 &61.30 &52.65  \\
    (OVA) &($\pm$3.30) &($\pm$4.95) &($\pm$2.01) &($\pm$4.42) &($\pm$5.96) &($\pm$4.39) &($\pm$3.80) &($\pm$3.36) &($\pm$3.63) 
    &($\pm$0.37) &($\pm$0.36) &($\pm$0.56) &($\pm$4.83)  \\
    CIFAR-10 &\bf{58.30} &49.04 &41.63 &35.68 &33.29 &27.62 &22.26 &22.36 &17.32 
    &\bf{69.11} &66.98 &62.10 &56.29  \\
    (PC) &($\pm$2.89) &($\pm$3.56) &($\pm$6.24) &($\pm$4.64) &($\pm$3.10) &($\pm$3.43) &($\pm$3.58) &($\pm$4.35) &($\pm$1.20) 
    &($\pm$2.59) &($\pm$3.72) &($\pm$1.47) &($\pm$2.36)  \\ \hline
  \end{tabular}
  }
  \label{tab:accuracy}
\end{table*}

\begin{table*}[t]\centering
\caption{
  Experimental classification errors for $10$ and $5$ class classification ($\times 10^{-3}$).
  The experiments were performed $5$ times for each case; the mean error and standard deviation are presented by the upper and lower values, respectively.
  The highest error is boldfaced.
}\scalebox{0.6}{
  \begin{tabular}{c|c c c c c c c c c|c c c c} \hline
    &&&&& $K=10$ &&&&&& $K=5$ \\
    N & 1 & 2  & 3 & 4 & 5 & 6 & 7 & 8 & 9 & 1 & 2 & 3 & 4 \\ \hline
    MNIST &0.27 &0.30 &0.19 &0.47 &0.72 &0.67 &1.96 &5.47 &\bf{10.69}
    &0.18 &0.24 &0.57 &\bf{4.07}  \\
    (OVA) &($\pm$0.19) &($\pm$0.19) &($\pm$0.18) &($\pm$0.29) &($\pm$0.40) &($\pm$0.34) &($\pm$0.99) &($\pm$3.92) &($\pm$5.80) 
    &($\pm$0.19) &($\pm$0.18) &($\pm$0.89) &($\pm$1.57)  \\
    MNIST &1.56 &1.41 &6.00 &4.94 &5.15 &8.22 &7.78 &35.94 &\bf{106.35} 
    &0.32 &1.40 &2.99 &\bf{6.74}  \\
    (PC) &($\pm$1.33) &($\pm$1.07) &($\pm$3.79) &($\pm$2.72) &($\pm$3.41) &($\pm$4.78) &($\pm$6.75) &($\pm$32.09) &($\pm$92.07) 
    &($\pm$0.23) &($\pm$0.94) &($\pm$1.15) &($\pm$2.20)  \\ \hline
    Fashion &0.72 &0.71 &1.08 &0.98 &1.36 &3.19 &4.98 &7.01 &\bf{19.51} 
    &0.36 &1.07 &1.75 &\bf{2.53}  \\
    (OVA) &($\pm$0.48) &($\pm$0.37) &($\pm$0.27) &($\pm$0.73) &($\pm$1.26) &($\pm$1.06) &($\pm$3.64) &($\pm$2.30) &($\pm$10.60) 
    &($\pm$0.25) &($\pm$0.61) &($\pm$1.17) &($\pm$1.76)\\
    Fashion &2.78 &3.86 &5.54 &9.55 &8.59 &16.56 &21.51 &37.45 &\bf{81.23} 
    &4.09 &3.99 &4.88 &\bf{12.46} \\
    (PC) &($\pm$2.30) &($\pm$3.70) &($\pm$2.21) &($\pm$3.70) &($\pm$8.09) &($\pm$10.85) &($\pm$11.12) &($\pm$28.52) &($\pm$50.82) 
    &($\pm$2.64) &($\pm$4.05) &($\pm$3.60) &($\pm$6.29)  \\ \hline
    Kuzushi &0.52 &0.56 &0.85 &1.11 &1.57 &1.66 &2.49 &6.01 &\bf{15.26} 
    &0.91 &1.46 &2.85 &\bf{6.10}  \\
    (OVA) &($\pm$0.37) &($\pm$0.45) &($\pm$0.58) &($\pm$0.45) &($\pm$0.30) &($\pm$0.63) &($\pm$1.85) &($\pm$3.20) &($\pm$7.67) 
    &($\pm$0.74) &($\pm$0.96) &($\pm$1.45) &($\pm$2.54)  \\
    Kuzushi &4.06 &5.73 &5.60 &9.27 &6.88 &20.39 &31.05 &27.94 &\bf{89.97} 
    &5.31 &2.94 &7.54 &\bf{24.13}  \\
    (PC) &($\pm$3.79) &($\pm$7.44) &($\pm$4.07) &($\pm$6.34) &($\pm$4.26) &($\pm$16.62) &($\pm$20.92) &($\pm$14.73) &($\pm$17.50) 
    &($\pm$1.61) &($\pm$2.95) &($\pm$6.36) &($\pm$20.90)  \\ \hline
    CIFAR-10 &0.48 &0.74 &0.67 &0.76 &1.47 &2.18 &3.54 &5.98 &\bf{6.33} 
    &1.01 &1.18 &2.79 &\bf{4.76}  \\
    (OVA) &($\pm$0.25) &($\pm$0.39) &($\pm$0.19) &($\pm$0.17) &($\pm$0.75) &($\pm$1.12) &($\pm$1.99) &($\pm$2.43) &($\pm$1.61) 
    &($\pm$0.66) &($\pm$0.74) &($\pm$1.54) &($\pm$2.97)  \\
    CIFAR-10 &5.02 &5.08 &7.51 &5.92 &9.39 &7.15 &23.80 &21.19 &\bf{86.60}
    &3.65 &4.43 &15.77 &\bf{25.66}  \\
    (PC) &($\pm$0.51) &($\pm$2.38) &($\pm$1.46) &($\pm$5.00) &($\pm$4.70) &($\pm$5.09) &($\pm$24.52) &($\pm$11.01) &($\pm$38.95) 
    &($\pm$0.72) &($\pm$2.40) &($\pm$7.76) &($\pm$8.08)  \\ \hline
  \end{tabular}
  }
  \label{tab:error}
\end{table*}
\subsection{Evaluation of Classification Error Bound}\label{subsec:data-generation}
{\bf Experimental Setup: }
Assuming that the number of classes $K$ is $10$ or $5$, the classification accuracy is compared under different numbers of candidate labels $N\in\left\{1, \cdots, K-1\right\}$.
The loss functions $\CandLoss[\text{OVA}]{}$ and $\CandLoss[\text{PC}]{}$ are defined using \eqref{eq:simple-CandLoss}, and the function $\loss{}$ is an origin-symmetric sigmoid loss, i.e., $\loss{z} = (1 + e^{z})^{-1} - 1/2$.
We set the number of training data for each class as $1,000$, and the data were randomly selected from the complete dataset.
We set the batch size as $64$ for the rest of the experiment.
We conducted the experiment using MLP for the MNIST, Fashion-MNIST, and Kuzushiji-MNIST datasets.
The number of epochs was $300$, weight decay was $10^{-4}$, learning rate was $5 \times 10^{-4}$, and \textit{Adam} was used as the optimization algorithm \cite{kingma2015}.
For CIFAR-10, the experiment was performed using DenseNet \cite{Huang2016cvpr}.
The number of epochs was $300$, weight decay was $5 \times 10^{-4}$, momentum was $0.9$, and the optimization algorithm was the \textit{stochastic gradient descent}.
The initial learning rate was set as $10^{-2}$, and it was halved every $30$ epochs.
The range of discriminant functions $g_{y}$ was $[-1/2, 1/2]$ for both MLP and DenseNet.

{\bf Error Computation: }
The classification error was calculated according to \eqref{eq:error}.
Because the exact $\hat{f}$ defined in Section \ref{subsec:notation} was not available, we surrogated it with a classifier that could minimize the loss of the training data in the experiment.
Similarly, we substituted $f^{*}$ with a classifier that could minimize the loss of the test data.

Additionally, we computed the theoretical classification error bounds according to \eqref{eq:ova-error} and \eqref{eq:pc-error}.
We set $\delta=0.1$ because Theorem~\ref{thm:errbound} holds with a high probability if $\delta$ is relatively small.
Furthermore, we set $\Lip$ as $1.0$ because it is the minimum Lipschitz constant of the shifted sigmoid loss.
Assuming that both MLP and DenseNet in the experiment had adequate capacity, we set $\Rad{\mathcal{G}}=0.5$.

{\bf Results: }
We performed $5$ trials for each experiment and observed the classification accuracy, that is, rate of correct classification for the test dataset estimated by a classifier which minimizes loss for training data.
The mean and standard deviation for the accuracy are listed in Table \ref{tab:accuracy}.
The results indicate that the mean of the accuracy tends to monotonically decrease corresponding to $N$.
That is, increase in the number of complementary labels leads to a better performance than the case of $N=K-1$ ($M=1$).

The mean and standard deviation of the experimental classification error are listed in Table \ref{tab:error}.
Furthermore, the mean of the experimental error and theoretical error bounds are shown in Figure \ref{fig:error} as a logarithmic graph.
The results indicate that the theoretical error bound for one-versus-all and pairwise classification are about equally tight.
The experimental error tends to monotonically increase corresponding to $N$, which is in accordance with the discussion in Section \ref{subsec:error-bound}.
The increase in theoretical error bound and experimental error are similar in shape, indicating that the bound reflects qualitative property found in the experimental error.
\section{Conclusion}\label{sec:conclusion}
Focusing on the fact that ordinary label and complementary label are essentially equivalent, we naturally related learning framework from single ordinary/complementary label.
We found that the loss functions for one-versus-all and pairwise classification satisfy certain \textit{additivity} and \textit{duality} which play a significant role in the analysis of classification risk and error bound.
As a result, we made it possible to understand learning from ordinary/complementary label, which are studied independently in a existing literature, as a new expression form in a unified perspective.

The analysis in this paper is under the the assumption that each training data is provided from a data-generation probability model which satisfies a certain uniformity of the assigned labels.
All the existing literature studying learning framework from multiple complementary labels are under the same assumption.
This assumption is natural in a context of privacy preserving.
As we described in this paper, the defined data-generation probability model indicates that providing multiple candidate labels generated from the model helps to prevent leakage of information about true distribution of labels for any patterns.
However, it is also interesting to analyze under the problem setting where the assumption does not hold.
Studies on the data-generation probability models to expand the range of application is the important future work.



\begin{thebibliography}{22}
\providecommand{\natexlab}[1]{#1}
\providecommand{\url}[1]{\texttt{#1}}
\expandafter\ifx\csname urlstyle\endcsname\relax
  \providecommand{\doi}[1]{doi: #1}\else
  \providecommand{\doi}{doi: \begingroup \urlstyle{rm}\Url}\fi

\bibitem[Cao and Xu(2020)]{Cao2020MultiComplementaryAU}
Yuzhou Cao and Yitian Xu.
\newblock Multi-complementary and unlabeled learning for arbitrary losses and
  models.
\newblock \emph{ArXiv}, abs/2001.04243, 2020.

\bibitem[Chapelle et~al.(2010)Chapelle, Schlkopf, and
  Zien]{chapelle2010mitpress}
Olivier Chapelle, Bernhard Schlkopf, and Alexander Zien.
\newblock \emph{Semi-Supervised Learning}.
\newblock The MIT Press, 1st edition, 2010.

\bibitem[Cid-sueiro(2012)]{jes2012nips}
Jes\'{u}s Cid-sueiro.
\newblock Proper losses for learning from partial labels.
\newblock In \emph{Advances in Neural Information Processing Systems 25}, pages
  1565--1573. 2012.

\bibitem[Cour et~al.(2011)Cour, Sapp, and Taskar]{cour2011jmlr}
Timothee Cour, Ben Sapp, and Ben Taskar.
\newblock Learning from partial labels.
\newblock \emph{J. Mach. Learn. Res.}, pages 1501--1536, 2011.

\bibitem[Feng et~al.(2020)Feng, Kaneko, Han, Niu, An, and
  Sugiyama]{Feng2020LearningFM}
Lei Feng, Takuo Kaneko, Bo~Han, Gang Niu, Bo~An, and Masashi Sugiyama.
\newblock Learning from multiple complementary labels.
\newblock \emph{ArXiv}, abs/1912.12927, 2020.

\bibitem[Grandvalet and Bengio(2005)]{grandvalet2005nips}
Yves Grandvalet and Yoshua Bengio.
\newblock Semi-supervised learning by entropy minimization.
\newblock In \emph{Advances in Neural Information Processing Systems 17}, pages
  529--536. MIT Press, 2005.

\bibitem[Huang et~al.(2016)Huang, Liu, and Weinberger]{Huang2016cvpr}
Gao Huang, Zhuang Liu, and Kilian~Q. Weinberger.
\newblock Densely connected convolutional networks.
\newblock \emph{2017 IEEE Conference on Computer Vision and Pattern
  Recognition}, pages 2261--2269, 2016.

\bibitem[Ishida et~al.(2017)Ishida, Niu, Hu, and Sugiyama]{ishida2017nips}
Takashi Ishida, Gang Niu, Weihua Hu, and Masashi Sugiyama.
\newblock Learning from complementary labels.
\newblock In \emph{Advances in Neural Information Processing Systems 30}, pages
  5639--5649. 2017.

\bibitem[Ishida et~al.(2018)Ishida, Niu, and Sugiyama]{ishida2018nips}
Takashi Ishida, Gang Niu, and Masashi Sugiyama.
\newblock Binary classification from positive-confidence data.
\newblock In \emph{Advances in Neural Information Processing Systems 31}, pages
  5917--5928. 2018.

\bibitem[Ishida et~al.(2019)Ishida, Niu, Menon, and Sugiyama]{ishida2019icml}
Takashi Ishida, Gang Niu, Aditya Menon, and Masashi Sugiyama.
\newblock Complementary-label learning for arbitrary losses and models.
\newblock In \emph{Proceedings of the 36th International Conference on Machine
  Learning}, pages 2971--2980, 2019.

\bibitem[Kim et~al.(2019)Kim, Yim, Yun, and Kim]{kim2019cvpr}
Youngdong Kim, Junho Yim, Juseung Yun, and Junmo Kim.
\newblock Nlnl: Negative learning for noisy labels.
\newblock In \emph{Proceedings of the IEEE International Conference on Computer
  Vision}, pages 101--110, 2019.

\bibitem[Kingma and Ba(2015)]{kingma2015}
Diederick~P Kingma and Jimmy Ba.
\newblock Adam: A method for stochastic optimization.
\newblock In \emph{In Proceedings of the International Conference on Learning
  Representations}, 2015.

\bibitem[Kipf and Welling(2017)]{kipf2016iclr}
Thomas~N. Kipf and Max Welling.
\newblock {Semi-Supervised Classification with Graph Convolutional Networks}.
\newblock In \emph{Proceedings of the 5th International Conference on Learning
  Representations}, 2017.

\bibitem[Laine and Aila(2017)]{laine2017iclr}
Samuli Laine and Timo Aila.
\newblock Temporal ensembling for semi-supervised learning.
\newblock In \emph{Proceedings of the International Conference on Learning
  Representations}, 2017.

\bibitem[Ledoux and Talagrand(2013)]{ledoux2013}
M.~Ledoux and M.~Talagrand.
\newblock \emph{Probability in Banach Spaces: Isoperimetry and Processes}.
\newblock Classics in Mathematics. Springer Berlin Heidelberg, 2013.

\bibitem[Luo and Orabona(2010)]{Jie2010nips}
Jie Luo and Francesco Orabona.
\newblock Learning from candidate labeling sets.
\newblock In \emph{Advances in Neural Information Processing Systems 23}, pages
  1504--1512. 2010.

\bibitem[Mann and McCallum(2007)]{mann2007icml}
Gideon~S. Mann and Andrew McCallum.
\newblock Simple, robust, scalable semi-supervised learning via expectation
  regularization.
\newblock In \emph{Proceedings of the 24th International Conference on Machine
  Learning}, page 593–600, 2007.

\bibitem[McDiarmid(1989)]{mcdiarmid1989}
Colin McDiarmid.
\newblock On the method of bounded differences.
\newblock In \emph{Surveys in Combinatorics}, London Mathematical Society
  Lecture Notes, pages 148--188. 1989.

\bibitem[Natarajan et~al.(2013)Natarajan, Dhillon, Ravikumar, and
  Tewari]{natarajan2013nips}
Nagarajan Natarajan, Inderjit~S Dhillon, Pradeep~K Ravikumar, and Ambuj Tewari.
\newblock Learning with noisy labels.
\newblock In \emph{Advances in Neural Information Processing Systems 26}, pages
  1196--1204. 2013.

\bibitem[Niu et~al.(2013)Niu, Jitkrittum, Dai, Hachiya, and
  Sugiyama]{gang2013pmlr}
Gang Niu, Wittawat Jitkrittum, Bo~Dai, Hirotaka Hachiya, and Masashi Sugiyama.
\newblock Squared-loss mutual information regularization: A novel
  information-theoretic approach to semi-supervised learning.
\newblock In \emph{Proceedings of the 30th International Conference on Machine
  Learning}, pages 10--18, 2013.

\bibitem[Raykar et~al.(2010)Raykar, Yu, Zhao, Valadez, Florin, Bogoni, and
  Moy]{raykar2010jmlr}
Vikas~C. Raykar, Shipeng Yu, Linda~H. Zhao, Gerardo~Hermosillo Valadez, Charles
  Florin, Luca Bogoni, and Linda Moy.
\newblock Learning from crowds.
\newblock \emph{J. Mach. Learn. Res.}, pages 1297--1322, 2010.

\bibitem[Zhang(2004)]{zhang2004jmlr}
Tong Zhang.
\newblock Statistical analysis of some multi-category large margin
  classification methods.
\newblock \emph{J. Mach. Learn. Res.}, pages 1225--1251, 2004.

\end{thebibliography}

\newpage
\appendix
\section{Proofs of Loss Function Properties}
As stated in Section \ref{sec:formulation}, if we define a loss function $\CandLoss{}$ in an additive form by \eqref{eq:CandLoss}, there exist constants $\xi_1$ and $\xi_2$ which satisfy \eqref{eq:ova-Cand-loss} and \eqref{eq:pc-Cand-loss}.
We first describe proof of \eqref{eq:ova-Cand-loss}.
The following formulation holds in accordance with $\loss{z}+\loss{-z}=a$.
\begin{align*}
    &\sum\limits_{y\in Y}\ordLoss[\text{OVA}]{f(x), y}\\
    &=\sum\limits_{y\in Y}\left(
    \loss{g_{y}(x)}+\frac{1}{K-1}\sum\limits_{y'\neq y}\loss{-g_{y'}(x)}
    \right)\\
    &=\sum\limits_{y\in Y}\left(
    \frac{K}{K-1}\loss{g_y(x)}-\frac{1}{K-1}\loss{g_y(x)}+\frac{1}{K-1}\sum\limits_{y'\in\mathcal{Y}}\loss{-g_{y'}(x)}-\frac{1}{K-1}\loss{-g_{y}(x)}
    \right)\\
    &=\frac{1}{K-1}\sum\limits_{y\in Y}\left(
    K\loss{g_{y}(x)}
    + \sum\limits_{y'\in\mathcal{Y}}\loss{-g_{y'}(x)}
    -a
    \right)\\
    &=\frac{1}{K-1}\left(
    K\sum_{y\in Y}\loss{g_{y}(x)}
    + \sum\limits_{y\in Y}\sum\limits_{y'\in\mathcal{Y}}\loss{-g_{y'}(x)}
    -\sum\limits_{y\in Y}a
    \right)\\
    &=\frac{1}{K-1}\left(
    (K-N+N)\sum\limits_{y\in Y}\loss{g_{y}(x)}
    +N\sum\limits_{y\in Y}\loss{-g_{y}(x)}
    +N\sum\limits_{y'\notin Y}\loss{-g_{y'}(x)}
    -a N
    \right)\\
    &=\frac{1}{K-1}\left(
    (K-N)\sum\limits_{y\in Y}\loss{g_{y}(x)}
    +N\sum\limits_{y'\notin Y}\loss{-g_{y'}(x)} +a N(N-1)
    \right)\\
    &=\frac{K-N}{K-1}\sum\limits_{y\in Y}\loss{g_{y}(x)}
    + \frac{N}{K-1}\sum\limits_{y'\notin Y}\loss{-g_{y'}(x)}
    + \frac{a N(N-1)}{K-1}\\
    &=\CandLoss[\text{OVA}]{f(x), Y} + \frac{a N(N-1)}{K-1}
\end{align*}
\qed

Next, we describe proof of \eqref{eq:pc-Cand-loss}.
Similar to the above formulation, in accordance with $\loss{z}+\loss{-z}=a$,
\begin{align*}
    &\sum\limits_{y\in Y}\ordLoss[\text{PC}]{f(x), y}\\
    &= \sum\limits_{y\in Y}\sum\limits_{y'\neq y}\loss{g_{y}(x)-g_{y'}(x)}\\
    &=\sum\limits_{y\in Y}\sum\limits_{y'\in\mathcal{Y}}\loss{g_{y}(x)-g_{y'}(x)}
    -\sum\limits_{y\in Y}\loss{g_{y}(x)-g_{y}(x)}\\
    &= \sum\limits_{y\in Y}\sum\limits_{y'\notin Y}\loss{g_y(x)-g_{y'}(x)}
    + \sum_{\substack{y, y'\in Y\\y=y'}}\loss{g_y(x)-g_{y'}(x)}
    + \sum_{\substack{y, y'\in Y\\y\neq y'}}\loss{g_y(x)-g_{y'}(x)}
    -\frac{a N}{2}\\
    &= \sum\limits_{y\in Y}\sum\limits_{y'\notin Y}\loss{g_y(x)-g_{y'}(x)}
    + a\cdot\comb{N}{2}\\
    &= \CandLoss[\text{PC}]{f(x), Y} + a\cdot\comb{N}{2}
\end{align*}
The third and fourth equality hold due to $\loss{0}=a/2$.
\qed
\section{Property of the Data-Generation Probability Model}
In this section, we describe proof of \eqref{eq:privacy}.
To begin with, we introduce the following lemma.
\begin{lemma}\label{lem:combination}
Let any finite sets be $X$ with size of $K$, and any elements of the $N$-size power set $\PowerSet{X}$ be $A$.
Then, the following equation holds for any function $f$ and $g$ over $X$.
\begin{align}
    \sum\limits_{A\in\PowerSet{X}}
    \sum\limits_{a_1, a_2\in A}f(a_1)g(a_2)
    &= \comb{K-2}{N-2}\sum_{
    \substack{x_1, x_2\in X\\x_1\neq x_2}}f(x_1)g(x_2)
    + \comb{K-1}{N-1}\sum_{
    \substack{x_1, x_2\in X\\x_1=x_2}}f(x_1)g(x_2)
    \label{eq:combination}
\end{align}
\end{lemma}

\begin{proof}
The left-hand side can be formulated as:
\begin{align*}
    \sum\limits_{A\in\PowerSet{X}}
    \sum\limits_{a_1, a_2\in A}f(a_1)g(a_2)
    &=\sum\limits_{A\in\PowerSet{X}}
    \sum_{\substack{a_1, a_2\in A\\a_1\neq a_2}}f(a_1)g(a_2)
    + \sum\limits_{A\in\PowerSet{X}}
    \sum_{\substack{a_1, a_2\in A\\a_1 = a_2}}f(a_1)g(a_2)
\end{align*}

For the first term, any $A$ can be chosen from $\PowerSet{X}$ in $\comb{K}{N}$ patterns, and any $a_1,~a_2~(a_1 \ne a_2)$ can be chosen from $A$ in ${}_{N}P_{2}$ patterns.
Because $f(a_1)g(a_2)$ with any $a_1,~a_2~(a_1 \ne a_2)$ can be chosen from $X$ in ${}_{K}P_{2}$ patterns, the first term in \eqref{eq:combination} holds due to $\comb{K-2}{N-2} = \frac{\comb{K}{N} \cdot {}_{N}P_{2}}{ {}_{K}P_{2}}$.
Similar for the second term.
\end{proof}

Here, we prove \eqref{eq:privacy}.
Let $P(\alpha\in Y|x)$ be the probability where a set of candidate labels $Y$ includes a label $\alpha\in\mathcal{Y}$.
Denoting $I$ as an indicator function, then
\begin{align*}
    P(\alpha\in Y)
    &= \sum\limits_{Y\in\PowerSet{\mathcal{Y}}}I(\alpha\in Y)\CandP{Y|x}\\
    &= \frac{1}{\comb{K-1}{N-1}}\sum\limits_{Y\in\PowerSet{\mathcal{Y}}}I(\alpha\in Y)\sum\limits_{y\in Y}P(y|x)\\
    &= \frac{1}{\comb{K-1}{N-1}}\sum\limits_{Y\in\PowerSet{\mathcal{Y}}}\sum\limits_{y_1, y_2\in Y}I(y_1=\alpha)P(y_2|x)\\
    &= \frac{1}{\comb{K-1}{N-1}}\left(\comb{K-1}{N-1}\sum\limits_{y\in\mathcal{Y}}I(y=\alpha)P(y|x) + \comb{K-2}{N-2}\sum_{
    \substack{y_1, y_2\in \mathcal{Y}\\y_1\neq y_2}}I(y_1=\alpha)P(y_2|x)\right)\\
    &= P(\alpha|x) + \frac{N-1}{K-1}\sum\limits_{y\neq\alpha}P(y|x)\\
    &= \frac{K-N}{K-1}P(\alpha|x)+\frac{N-1}{K-1}
\end{align*}
The fourth equality holds due to lemma~\ref{lem:combination}.
In the fifth equality, the second term holds because $y_2\neq\alpha$ holds if $y_1=\alpha$. 
Therefore, 
\begin{align*}
    Q(y_o=\alpha|x) &= Q(y_o=\alpha|x, \alpha\in Y)Q(\alpha\in Y|x)
    + Q(y_o=\alpha|x, \alpha\notin Y)Q(\alpha\notin Y|x)\\
    &= \frac{1}{N}\frac{K-N}{N-1}P(\alpha|x)+\frac{1}{N}\frac{N-1}{K-1}\\
    &= \beta P(\alpha|x) + (1-\beta)\frac{1}{K}
\end{align*}
The second equality holds due to the assumption where $Q(y_o=\alpha|x, \alpha\in Y)=1/N$ and $Q(y_o=\alpha|x, \alpha\notin Y)=0$ hold.
\qed
\section{Proof of Theorem \ref{thm:risk}}
We first derive the expectation of the sum of loss for all the ordinary labels.
\begin{align*}
    &\Expected[\CandP{Y|x}]{\sum\limits_{y\in Y}\ordLoss{f(x), y}}\\
    &=\sum\limits_{Y\in\PowerSet{\mathcal{Y}}}
    \sum\limits_{y\in Y}\ordLoss{f(x), y}\CandP{Y|x}\\
    &= \frac{1}{\comb{K-1}{N-1}}
    \sum\limits_{Y\in\PowerSet{\mathcal{Y}}}
    \sum\limits_{y\in Y}
    \sum\limits_{y'\in Y}
    \ordLoss{f(x), y'}\CandP{y|x}\\
    &=\frac{1}{\comb{K-1}{N-1}}\left\{
    \comb{K-2}{N-2}\sum\limits_{y\in\mathcal{Y}}\sum\limits_{y'\neq y}
    \ordLoss{f(x), y'}\CandP{y|x}
    +\comb{K-1}{N-1}\sum\limits_{y\in\mathcal{Y}}
    \ordLoss{f(x), y}\CandP{y|x}
    \right\}\\
    &= \frac{N-1}{K-1}\sum\limits_{y\in\mathcal{Y}}\sum\limits_{y'\neq y}
    \ordLoss{f(x), y'}\CandP{y|x}
    +\sum\limits_{y\in\mathcal{Y}}
    \ordLoss{f(x), y}\CandP{y|x}\\
    &= \frac{N-1}{K-1}\Expected[\OrdP{y|x}]{
    \sum\limits_{y'\in \mathcal{Y}}\ordLoss{f(x), y'}
    -\ordLoss{f(x), y}}
    + \Expected[\OrdP{y|x}]{
    \ordLoss{f(x), y}}\\
    &= \frac{N-1}{K-1}
    \sum\limits_{y\in \mathcal{Y}}\ordLoss{f(x), y}
    + \frac{K-N}{K-1}
    \Expected[\OrdP{y|x}]{\ordLoss{f(x), y}}
\end{align*}
The second equality holds due to the definition of $\CandP{Y|x}$.
The third equality holds due to Lemma~\ref{lem:combination}.
Thus, the following formulation holds due to \eqref{eq:CandLoss}.
\begin{align*}
    \Expected[\CandP{x, Y}]{\CandLoss{f(x), Y}}
    &= \xi_1\Expected[\CandP{x, Y}]{
    \sum\limits_{y\in Y}\ordLoss{f(x), y}}+\xi_2\\
    &= \frac{\xi_1 (K-N)}{K-1}
    \Expected[\OrdP{x, y}]{\ordLoss{f(x), y}}
    + \frac{\xi_1 (N-1)}{K-1}
    \sum\limits_{y\in\mathcal{Y}}\ordLoss{f(x), y}
    +\xi_2
\end{align*}
\qed

\section{Proof of Lemma \ref{lem:candloss-norm}}

According to the duality described in \eqref{eq:CandLoss} and \eqref{eq:CompLoss}, loss function $\CandLoss{}$ can be formulated by $\ordLoss{f(x), y}$ for $N$ candidate labels $y\in Y$, or $\compLoss{f(x), \overline{y}}$ for $K-N$ complementary labels $\overline{y}\in\overline{Y}$.
Thus, we can redefine loss function $\CandLoss{}$ as the following:
\begin{align*}
    \CandLoss{f(x), Y}
    &=\sum\limits_{y\in \widetilde{Y}}\tildeLoss{f(x), y}
    =\begin{cases}
    \sum\limits_{y\in Y}\ordLoss{f(x), y}, & \text{if $N\le \dfrac{K}{2}$}\\[.5em]
    \sum\limits_{\overline{y}\in\overline{Y}}\compLoss{f(x), \overline{y}}, & \text{otherwise}\\
    \end{cases}
\end{align*}
where $\tildeLoss{}$ and $\widetilde{Y}$ denote $\ordLoss{}$ and $Y$ respectively if $N\leq K/2$, otherwise $\compLoss{}$, $\overline{Y}$ respectively.
Therefore, given $\widetilde{N}:=|\widetilde{Y}|$ it always satisfies $\widetilde{N}\leq K/2$.
Note that $\xi_{2} = \overline{\xi}_{2} = 0$ due to the assumption of $a=0$, as discussed in Section \ref{subsec:loss-function}.
Similarly, $\tildeloss{z}$ denotes $\tildeloss{}: z\mapsto\loss{z}$ if $N\leq K/2$ otherwise $\tildeloss{}: z\mapsto\loss{-z}$.
For the rest of this work, we prove Lemma~\ref{lem:candloss-norm} according to those definitions.

First we describe proof for one-versus-all classification.
Under the assumption of $a=0$, the following formulation holds.
\begin{align*}
    \sum\limits_{y\in\widetilde{Y}}\tildeLoss[\text{OVA}]{f(x), y}
    &=\sum\limits_{y\in \widetilde{Y}}\left(
    \frac{K}{K-1}\tildeloss{g_{y}(x)} + \frac{1}{K-1}\sum\limits_{y'\in\mathcal{Y}}\tildeloss{-g_{y'}(x)}
    \right)
\end{align*}
Thus, the following formulation holds for any $\widetilde{Y}, \widetilde{Y}'\in\PowerSet{\mathcal{Y}}$.
\begin{align*}
    \|\CandLoss[\text{OVA}]{}\|_{\infty}
    &=\sup\limits_{\seq{g}{K}\in\mathcal{G}}\left(
    \sum\limits_{y\in \widetilde{Y}}\tildeLoss[{\text{OVA}}]{f(x), y}
    - \sum\limits_{y'\in \widetilde{Y}'}\tildeLoss[{\text{OVA}}]{f(x), y'}\right)\\
    &= \sup\limits_{\seq{g}{K}\in\mathcal{G}}\Biggl\{
    \frac{K}{K-1}\left(
    \sum\limits_{y\in \widetilde{Y}}\tildeloss{g_{y}(x)}
    - \sum\limits_{y'\in \widetilde{Y}'}\tildeloss{g_{y'}(x)}
    \right)\\
    &\hspace{1em}+\frac{1}{K-1}\left(
    \sum\limits_{y\in\mathcal{Y}}\tildeloss{-g_{y}(x)}
    - \sum\limits_{y\in\mathcal{Y}}\tildeloss{-g_{y}(x)}
    \right)\Biggr\}\\
    &\leq\frac{K}{K-1}\left(\sup\limits_{\seq{g}{K}\in\mathcal{G}}
    \sum\limits_{y\in \widetilde{Y}}\tildeloss{g_{y}(x)}
    - \inf\limits_{\seq{g}{K}}\sum\limits_{y'\in \widetilde{Y}'}\tildeloss{g_{y'}(x)}
    \right)\\
    &\leq\frac{K}{K-1}\left(
    \frac{\widetilde{N}}{2} + \frac{\widetilde{N}}{2}
    \right)\\
    &= \frac{K\widetilde{N}}{K-1}
\end{align*}
The second inequality holds because supremum and infimum of $\loss{}$ are $1/2$ and $-1/2$ respectively.
\qed

We further describe proof for pairwise classification.
Under the assumption of $a=0$, the following formulation holds.
\begin{align*}
    \sum\limits_{y\in\widetilde{Y}}\tildeLoss[\text{PC}]{f(x), y}
    &= \sum\limits_{y\in \widetilde{Y}}\sum\limits_{y'\notin \widetilde{Y}}\tildeloss{g_y(x)-g_{y'}(x)}
\end{align*}
Thus, 
\begin{align*}
    \|\CandLoss[\text{PC}]{}\|_{\infty}
    &=\sup\limits_{\seq{g}{K}\in\mathcal{G}}\left(
    \sum\limits_{y'\in \widetilde{Y}}\tildeLoss[{\text{PC}}]{f(x), y'}
    - \sum\limits_{y\in \widetilde{Y}'}\tildeLoss[{\text{PC}}]{f(x), y}\right)\\
    &=\sup\limits_{\seq{g}{K}\in\mathcal{G}}
    \sum\limits_{y\in \widetilde{Y}}\sum\limits_{y'\notin \widetilde{Y}}\tildeloss{g_y(x)-g_{y'}(x)}
    - \inf\limits_{\seq{g}{K}\in\mathcal{G}}
    \sum\limits_{y'\in \widetilde{Y}'}\sum\limits_{y\notin \widetilde{Y}'}\tildeloss{g_{y'}(x)-g_y(x)}\\
    &\leq\frac{\widetilde{N}(K-\widetilde{N})}{2}-\left(-\frac{\widetilde{N}(K-\widetilde{N})}{2}\right)\\
    &= \widetilde{N}(K-\widetilde{N})
\end{align*}
\qed

\section{Proof of Lemma \ref{lem:rademacher-bound}}
We decribe proof for one-versus-all classification.
Under the assumption that $h\in\mathcal{H}_{\text{OVA}}$ is equivalent to $\CandLoss[\text{OVA}]{}$, the following formulation holds due to the definition of $\mathcal{H}_{\text{OVA}}$.
\begin{align*}
    \Rad{\mathcal{H}_{{\text{OVA}}}}
    &= \Expected[\mathcal{S}][\mathcal{\sigma}]{
    \sup\limits_{h\in\mathcal{H}_{\text{OVA}}}\frac{1}{n}
    \sum_{i=1}^{n}
    \sigma_i h\left(x_i, Y_i\right)
    }\\
    &= \Expected[\mathcal{S}][\mathcal{\sigma}]{
    \sup\limits_{\seq{g}{K}\in\mathcal{G}}\frac{1}{n}
    \sum_{i=1}^{n}\sigma_i
    \sum\limits_{y\in \widetilde{Y}_i}\tildeLoss[{\text{OVA}}]{f(x_i), y}}\\
    &= \Expected[\mathcal{S}][\mathcal{\sigma}]{
    \sup\limits_{\seq{g}{K}\in\mathcal{G}}\frac{1}{n}
    \sum\limits_{i=1}^{n}
    \sigma_i \sum\limits_{y\in \widetilde{Y}_i}\left\{
    \frac{K}{K-1}\tildeloss{g_y(x_i)}
    +\frac{1}{K-1}\sum\limits_{y'\in \mathcal{Y}}\tildeloss{-g_{y'}(x_i)}
    \right\}}\\
    &\leq \frac{K}{K-1}\Expected[\mathcal{S}][\mathcal{\sigma}]{
    \sup\limits_{\seq{g}{K}\in\mathcal{G}}\frac{1}{n}
    \sum\limits_{i=1}^{n}
    \sigma_i \sum\limits_{y\in \widetilde{Y}_i}
    \tildeloss{g_y(x_i)}}\\
    &+\frac{1}{K-1}\Expected[\mathcal{S}][\mathcal{\sigma}]{
    \sup\limits_{\seq{g}{K}\in\mathcal{G}}\frac{1}{n}
    \sum\limits_{i=1}^{n}
    \sigma_i \sum\limits_{y\in \widetilde{Y}_i}
    \sum\limits_{y'\in \mathcal{Y}}\tildeloss{-g_{y'}(x_i)}
    }
\end{align*}
Let $I(y\in\widetilde{Y}_{i})$ be an indicator function and define $\alpha_{i} := 2 I(y\in\widetilde{Y}_{i}) - 1$, then for the first term, 
\begin{align*} 
    &\Expected[\mathcal{S}][\mathcal{\sigma}]{
    \sup\limits_{\seq{g}{K}\in\mathcal{G}}\frac{1}{n}
    \sum_{i=1}^{n}
    \sigma_i \sum\limits_{y\in \widetilde{Y}_i}
    \tildeloss{g_y(x_i)}}\\
    &=\Expected[\mathcal{S}][\mathcal{\sigma}]{
    \sup\limits_{\seq{g}{K}\in\mathcal{G}}\frac{1}{n}
    \sum_{i=1}^{n}
    \sigma_i \sum\limits_{y\in \mathcal{Y}}
    \tildeloss{g_y(x_i)}I(y\in\widetilde{Y}_{i})}\\
    &=\Expected[\mathcal{S}][\mathcal{\sigma}]{
    \sup\limits_{\seq{g}{K}\in\mathcal{G}}\frac{1}{2 n}
    \sum_{i=1}^{n}
    \sigma_i \sum\limits_{y\in \mathcal{Y}}
    \tildeloss{g_y(x_i)}(\alpha_{i} + 1)}\\
    &\leq \sum\limits_{y\in \mathcal{Y}}\Expected[\mathcal{S}][\mathcal{\sigma}]{
    \sup\limits_{\seq{g}{K}\in\mathcal{G}}\frac{1}{2 n}
    \sum_{i=1}^{n}
    \sigma_i \tildeloss{g_y(x_i)}(\alpha_{i} + 1)}\\
    &\leq \sum\limits_{y\in\mathcal{Y}}\left\{
    \Expected[\mathcal{S}][\mathcal{\sigma}]{
    \sup\limits_{\seq{g}{K}\in\mathcal{G}}\frac{1}{2 n}
    \sum_{i=1}^{n}
    \alpha_{i}\sigma_i \sum\limits_{y\in \mathcal{Y}}
    \tildeloss{g_y(x_i)}}
    + \Expected[\mathcal{S}][\mathcal{\sigma}]{
    \sup\limits_{\seq{g}{K}\in\mathcal{G}}\frac{1}{2 n}
    \sum_{i=1}^{n}
    \sigma_i \sum\limits_{y\in \mathcal{Y}}
    \tildeloss{g_y(x_i)}}
    \right\}\\
    &= K\Expected[\mathcal{S}][\mathcal{\sigma}]{
    \sup\limits_{g\in\mathcal{G}}\frac{1}{n}
    \sum_{i=1}^{n}
    \sigma_i \sum\limits_{y\in \mathcal{Y}}
    \tildeloss{g(x_i)}}\\
    &= K \Rad{\tildeloss{}\circ\mathcal{G}}
\end{align*}
The second equality from the last holds because $\sigma_{i}$ and $\alpha_{i}\sigma_{i}$ are drawn from the same probabilistic distribution.
For the second term, 
\begin{align*}
    \Expected[\mathcal{S}][\mathcal{\sigma}]{
    \sup\limits_{\seq{g}{K}\in\mathcal{G}}\frac{1}{n}
    \sum_{i=1}^{n}
    \sigma_i \sum\limits_{y\in \widetilde{Y}_i}
    \sum\limits_{y'\in \mathcal{Y}}\tildeloss{-g_{y'}(x_i)}}
    &\leq \sum\limits_{y\in\mathcal{Y}}
    \Expected[\mathcal{S}][\mathcal{\sigma}]{
    \sup\limits_{g_{y}\in\mathcal{G}}\frac{1}{n}
    \sum_{i=1}^{n}
    \sigma_i \widetilde{N}\tildeloss{-g_{y}(x_i)}}\\
    &= K \widetilde{N} 
    \Expected[\mathcal{S}][\mathcal{\sigma}]{
    \sup\limits_{g\in\mathcal{G}}\frac{1}{n}
    \sum_{i=1}^{n}
    \sigma_i \tildeloss{-g(x_i)}}\\
    &= K \widetilde{N} \Rad{\tildeloss{}\circ\mathcal{G}}
\end{align*}
Thus, 
\begin{align*}
    \Rad{\mathcal{H}_{\text{OVA}}}
    &\leq \frac{K^{2} + K\widetilde{N}}{K-1}\Rad{\tildeloss{}\circ\mathcal{G}}\\
    &\leq \frac{K(K + \widetilde{N})}{K-1}\Lip\Rad{\mathcal{G}}
\end{align*}
The second inequality holds due to $\Rad{\tildeloss{}\circ\mathcal{G}}\leq\Lip\Rad{\mathcal{G}}$ according to Talagrand's contraction lemma \cite{ledoux2013}.
\qed

Note that $\loss{z}-\loss{-z}=a$ is incorrectly assumed in \cite{ishida2017nips}, which causes miscalculation in proof of Lemma~\ref{lem:rademacher-bound}.

We further describe proof for pairwise classification.
Under the assumption that $h$ and $\CandLoss[\text{PC}]{}$ are equivalent,
\begin{align*}
    &\Rad{\mathcal{H}_{\text{PC}}}\\
    &= \Expected[\mathcal{S}][\mathcal{\sigma}]{
    \sup\limits_{h\in\mathcal{H}_{\text{PC}}}\frac{1}{n}
    \sum_{i=1}^{n}
    \sigma_i h\left(x_i, Y_i\right)
    }\\
    &=\Expected[\mathcal{S}][\mathcal{\sigma}]{
    \sup\limits_{\seq{g}{K}\in\mathcal{G}}\frac{1}{n}\sum_{i=1}^{n}
    \sigma_{i}\sum\limits_{y\in \widetilde{Y}_i}\tildeLoss[\text{PC}]{f(x_i), y}}\\
    &=\Expected[\mathcal{S}][\mathcal{\sigma}]{
    \sup\limits_{\seq{g}{K}\in\mathcal{G}}\frac{1}{n}\sum_{i=1}^{n}
    \sigma_{i}\sum\limits_{y\in \widetilde{Y}_i}\sum\limits_{y'\neq y}\tildeloss{g_{y}(x_i)-g_{y'}(x_i)}
    }\\
    &=\Expected[\mathcal{S}][\mathcal{\sigma}]{
    \sup\limits_{\seq{g}{K}\in\mathcal{G}}\frac{1}{n}\sum_{i=1}^{n}
    \sigma_{i}\sum\limits_{y\in \mathcal{Y}}\sum\limits_{y'\neq y}
    \tildeloss{g_{y}(x_i)-g_{y'}(x_i)}I(y\in\widetilde{Y}_{i})
    }\\
    &\leq\sum\limits_{y\in \mathcal{Y}}\sum\limits_{y'\neq y}
    \Expected[\mathcal{S}][\mathcal{\sigma}]{
    \sup\limits_{g_{y}, g{y'}\in\mathcal{G}}\frac{1}{2 n}\sum_{i=1}^{n}
    \sigma_{i}\tildeloss{g_{y}(x_i)-g_{y'}(x_i)}(\alpha_{i} + 1)
    }\\
    &\leq\sum\limits_{y\in \mathcal{Y}}\sum\limits_{y'\neq y}
    \Biggl\{
    \Expected[\mathcal{S}][\mathcal{\sigma}]{
    \sup\limits_{g_{y}, g{y'}\in\mathcal{G}}\frac{1}{2 n}\sum_{i=1}^{n}
    \alpha_{i}\sigma_{i}\tildeloss{g_{y}(x_i)-g_{y'}(x_i)}}\\
    &\hspace{1em}+\Expected[\mathcal{S}][\mathcal{\sigma}]{
    \sup\limits_{g_{y}, g{y'}\in\mathcal{G}}\frac{1}{2 n}\sum_{i=1}^{n}
    \sigma_{i}\tildeloss{g_{y}(x_i)-g_{y'}(x_i)}}
    \Biggr\}\\
    &\leq\sum\limits_{y\in \mathcal{Y}}\sum\limits_{y'\neq y}
    \Expected[\mathcal{S}][\mathcal{\sigma}]{
    \sup\limits_{g_{y}, g_{y'}\in\mathcal{G}}\frac{1}{n}\sum_{i=1}^{n}
    \sigma_{i}\tildeloss{g_{y}(x_i)-g_{y'}(x_i)}}
\end{align*}
Let we define $\mathcal{G}_{g_{y}, g_{y'}}:=\left\{x\mapsto g_{y}(x)-g_{y'}(x)|g_{y}, g_{y'}\in\mathcal{G}\right\}$, then:
\begin{align*}
    &\Expected[\mathcal{S}][\mathcal{\sigma}]{
    \sup\limits_{g_{y}, g_{y'}\in\mathcal{G}}\frac{1}{n}\sum_{i=1}^{n}
    \sigma_{i}\tildeloss{g_{y}(x_i)-g_{y}(x_i)}}\\
    &= \Rad{\tildeloss{}\circ\mathcal{G}_{g_{y}, g_{y'}}}\\
    &\leq \Lip\Rad{\mathcal{G}_{g_{y}, g_{y'}}}\\
    &=\Lip\Expected[\mathcal{S}][\mathcal{\sigma}]{
    \sup\limits_{g_{y}, g_{y'}\in\mathcal{G}}
    \frac{1}{n}\sum_{i=1}^{n}\sigma_{i}\left(g_{y}(x_i)-g_{y'}(x_i)\right)}\\
    &\leq\Lip\Expected[\mathcal{S}][\mathcal{\sigma}]{
    \sup\limits_{g_{y}\in\mathcal{G}}
    \frac{1}{n}\sum_{i=1}^{n}\sigma_{i}g_{y}(x_i)}
    +\Lip\Expected[\mathcal{S}][\mathcal{\sigma}]{
    \sup\limits_{g_{y'}\in\mathcal{G}}
    \frac{1}{n}\sum_{i=1}^{n}(-\sigma_{i})g_{y'}(x_i)}\\
    &=2\Lip\Rad{\mathcal{G}}
\end{align*}
The third equality holds because $\sigma_{i}$ and $-\sigma_{i}$ are drawn from the same probabilistic distribution.
Then,
\begin{align*}
    \Rad{\mathcal{H}_{\text{PC}}}
    &\leq 2 K(K-1)\Lip\Rad{\mathcal{G}}
\end{align*}
\qed

\section{Proof of Theorem \ref{thm:errbound}}

We only describe proof for one-versus-all classification; proof for pairwise classification is similar.
We substitute the $j$~th data $(x_j, Y_j)$ in $\mathcal{S}$ with any data $(x'_j, Y'_j)$, and define the data set as $\mathcal{S'}$.
Let a set of empirical discrimination functions and empirical risk for $\mathcal{S'}$ be $\mathcal{G'}:=\left\{g'\right\}$ and $\hat{R'}(f)$ respectively.
Then the following formulation holds due to Lemma~\ref{lem:candloss-norm}.
\begin{align*}
    &\sup\limits_{\seq{g}{K}\in\mathcal{G}}\left(\hat{R}(f)-R(f)\right)
    - \sup\limits_{\seq{g'}{K}\in\mathcal{G'}}\left(\hat{R'}(f)-R(f)\right)\\
    &\leq\frac{K-1}{n(K-N)}
    \sup_{\seq{g}{K}\in\mathcal{G}}\inf_{\seq{g'}{K}\in\mathcal{G'}}\left\{
    \left(\CandLoss[\text{OVA}]{f(x_j), Y_j}
    - \CandLoss[\text{OVA}]{f({x'}_j), {Y'}_j}\right)
    \right\}\\
    &\leq\frac{K-1}{n(K-N)}
    \|\CandLoss[\text{OVA}]{}\|_{\infty}\\
    &\leq\frac{K\widetilde{N}}{n(K-N)}
\end{align*}
According to McDiarmid's inequality\cite{mcdiarmid1989}, for any integer $\delta>0$ the following formulation holds with a probability at least $1-\delta/2$.
\begin{align*}
    \sup\limits_{\seq{g}{K}}\left(\hat{R}(f)-R(f)\right)
    -\Expected{\sup\limits_{\seq{g}{K}}\left(\hat{R}(f)-R(f)\right)}
    &\leq\frac{K\widetilde{N}}{2(K-N)}\sqrt{\frac{2\ln(2/\delta)}{n}}
\end{align*}
Let $\mathcal{S}':=\left\{(x'_{i}, Y'_{i})\right\}_{i=1}^{n}$ be any dataset where each data is drawn from the data-generation probability model $\CandP{x, Y}$.
Due to $R(f)=\Expected{\hat{R}(f)}$, 
\begin{align*}
    &\Expected{\sup\limits_{\seq{g}{K}\in\mathcal{G}}\left(\hat{R}(f)-R(f)\right)}\\
    &=\frac{K-1}{K-N}\Expected[\mathcal{S}]{\sup_{\seq{g}{K}\in\mathcal{G}}\left(
    \frac{1}{n}\sum_{i=1}^{n}\CandLoss[\text{OVA}]{f(x_i), Y_i}
    -\Expected[\mathcal{S}']{\frac{1}{n}\sum_{i=1}^{n}\CandLoss[\text{OVA}]{f(x'_i), Y'_i}}\right)}\\
    &\leq\frac{K-1}{K-N}\Expected[\mathcal{S}][\mathcal{S}']{
    \sup\limits_{\seq{g}{K}\in\mathcal{G}}\left(
    \frac{1}{n}\sum_{i=1}^{n}\CandLoss[\text{OVA}]{f(x_i), Y_i}
    -\frac{1}{n}\sum_{i=1}^{n}(\CandLoss[\text{OVA}]{f(x'_i), Y'_i}
    \right)}\\
    &=\frac{K-1}{K-N}\Expected[\mathcal{S}][\mathcal{S}'][\mathcal{\sigma}]{
    \sup\limits_{\seq{g}{K}\in\mathcal{G}}\left(
    \frac{1}{n}\sum_{i=1}^{n}\sigma_{i}\CandLoss[\text{OVA}]{f(x_i), Y_i}
    +\frac{1}{n}\sum_{i=1}^{n}(-\sigma_{i})\CandLoss[\text{OVA}]{f(x'_i), Y'_i}
    \right)}\\
    &\leq\frac{2(K-1)}{K-N}\Expected[\mathcal{S}][\mathcal{\sigma}]{
    \sup\limits_{\seq{g}{K}\in\mathcal{G}}
    \frac{1}{n}\sum_{i=1}^{n}\sigma_{i}\CandLoss[\text{OVA}]{f(x_i), Y_i}
    }\\
    &=\frac{2(K-1)}{K-N}\Rad{\mathcal{H}_{\text{OVA}}}\\
    &\leq\frac{2 K (K + \widetilde{N})}{K-N}\Lip\Rad{\mathcal{G}}
\end{align*}
The second equality holds because $\CandLoss{f(x_{i}, Y_{i})}$ and $\sigma_{i}\CandLoss{f(x_{i}, Y_{i})}$ are drawn from the same probabilistic distribution; similar for $\CandLoss{f(x'_{i}, Y'_{i})}$.
The last inequality holds due to Lemma~\ref{lem:rademacher-bound}.
$\hat{R}(\hat{f})\leq\hat{R}(f^{*})$ holds according to \eqref{eq:error}, thus,
\begin{align*}
    \err
    &=\left(\hat{R}(\hat{f})-\hat{R}(f^{*})\right)
    +\left(R(\hat{f})-\hat{R}(\hat{f})\right)
    +\left(\hat{R}(f^{*})-R(f^{*})\right)\\
    &\leq 2\sup\limits_{\seq{g}{K}\in\mathcal{G}}
    \left|\hat{R}(f)-R(f)\right|\\
    &\leq\frac{4 K (K + \widetilde{N})}{K-N}\Lip\Rad{\mathcal{G}}
    + \frac{K\widetilde{N}}{K-N}\sqrt{\frac{2\ln(2/\delta)}{n}}
\end{align*}
\qed

\end{document}